\documentclass[11pt]{article}

\usepackage{fullpage,times,url,natbib}

\usepackage{amsthm,amsfonts,amsmath,amssymb,epsfig,color,float,graphicx,verbatim}
\usepackage{algorithm,algorithmic}

\usepackage{hyperref}
\hypersetup{
	colorlinks   = true, %Colours links instead of ugly boxes
	urlcolor     = blue, %Colour for external hyperlinks
	linkcolor    = blue, %Colour of internal links
	citecolor   = black %Colour of citations
}

\newtheorem{theorem}{Theorem}

\newtheorem{lemma}{Lemma}

\newcommand{\reals}{\mathbb{R}}
\newcommand{\E}{\mathbb{E}}

\newcommand{\be}{\mathbf{e}}
\newcommand{\bx}{\mathbf{x}}
\newcommand{\bw}{\mathbf{w}}

\newcommand{\bsigma}{\boldsymbol{\sigma}}

\newcommand{\Ocal}{\mathcal{O}}

\newcommand{\Dcal}{\mathcal{D}}

\newcommand{\Wcal}{\mathcal{W}}

\newcommand{\norm}[1]{\|#1\|}
\newcommand{\inner}[1]{\langle#1\rangle}

\renewcommand{\eqref}[1]{Eq.~(\ref{#1})}
\newcommand{\lemref}[1]{Lemma~\ref{#1}}

\newcommand{\thmref}[1]{Thm.~\ref{#1}}

\title{The Sample Complexity of Learning Linear Predictors\\with the Squared Loss}
\author{Ohad Shamir\\Weizmann Institute of Science\\\texttt{ohad.shamir@weizmann.ac.il}}
\date{}

\begin{document}

\maketitle

\begin{abstract}
  In this short note, we provide a sample complexity lower bound for
  learning linear predictors with respect to the squared loss. Our focus is
  on an agnostic setting, where no assumptions are made on the data
  distribution. This contrasts with standard results in the literature,
  which either make distributional assumptions, refer to specific parameter
  settings, or use other performance measures.
\end{abstract}

\section{Introduction}

In machine learning and statistics, the squared loss is the most commonly
used loss for measuring real-valued predictions: Given a prediction $p$ and
actual target value $y$, it is defined as $\ell(p,y)=(p-y)^2$. It is
intuitive, has a convenient analytical form, and has been extremely
well-studied.

In this note, we concern ourselves with learning linear predictors with
respect to the squared loss, in a standard agnostic learning framework.
Formally, for some fixed parameters $X,Y,B$, we assume the existence of an
unknown distribution over $\{\bx\in \reals^d:\norm{\bx}\leq
1\}\times\{y\in\reals:|y|\leq Y\}$, from which we are given a training set
$S=\{\bx_i,y_i\}_{i=1}^{m}$ of $m$ i.i.d. examples, consisting of pairs of
instances $\bx$ and target values $y$. Given a linear predictor $\bx\mapsto
\inner{\bw,\bx}$, its risk with respect to the squared loss is defined as
\[
R(\bw)=\E_{(\bx,y)}[(\inner{\bw,\bx}-y)^2].
 \]
Our goal is to find a linear predictor $\bw$ from the hypothesis class of
norm-bounded linear predictors,
\[
\Wcal=\{\bw:\norm{\bw}\leq B\},
\]
such that its excess risk
\[
R(\bw)-\min_{\bw\in\Wcal}R(\bw)
\]
with respect to the best possible predictor in $\Wcal$ is as small as
possible. We focus here on the expected excess risk (over the randomness of
the training set and algorithm), and consider how it is affected by the
problem parameters $Y,B,d$ and the sample size $m$, uniformly over any
distribution.

Despite a huge literature on learning with the squared loss, we were unable
to locate an explicit and self-contained analysis for this question. The
existing results (some examples include
\citet{hsu2014random,koltchinskii2011oracle,lecue2014performance,tsybakov2003optimal,anthony1999neural,lee1998importance})
all appear to differ from our setting in one or more of the following
manners:
\begin{itemize}
  \item \emph{Distributional Assumptions:} In our agnostic setting, we
      assume nothing whatsoever about the data distribution, other than
      boundedness (as specified by $X,Y$). In contrast, most existing works
      rely on additional assumptions. Perhaps the most common assumption is
      a well-specified model, under which there exists a fixed
      $\bw\in\reals^d$ such that $y=\inner{\bw,\bx}+\xi$, where $\xi$ is a
      zero-mean noise term. Other works impose some moment or other
      conditions on the distribution of $\bx$, or consider a fixed design
      setting where the data instances are not sampled i.i.d.. These
      assumptions usually lead to excess risk bounds which scale (at least
      in finite dimensions) as $dY^2/m$, independent of the norm bound $B$.
      However, as we will see later, this is not the behavior in the
      distribution-free setting.
  \item \emph{Bounds not on the excess risk:} Many of the existing results
      are not on the excess risk, but rather on  $\E[\norm{\bw-\bw^*}^2]$
      or $\E[(\inner{\bw,\bx}-\inner{\bw^*,\bx})^2]$, where
      $\bw^*=\arg\min_{\bw\in\Wcal} R(\bw)$. The former measure is relevant
      for parameter estimation, while the latter measure can be shown to
      equal the excess risk when $\bw^*=\arg\min_{\bw\in\reals^d}R(\bw)$
      (i.e. $B=\infty$ - see \lemref{lem:proj} below). However, when we
      deal with the hypothesis class of norm-bounded predictors, then the
      excess risk can be larger by an arbitrary factor\footnote{For
      example, consider a distribution on $(x,y)$ such that $(x,y)=(1,1)$
      with probability $1$, and $\Wcal=\{w:w\in [-1/2,1/2]\}$. Then
      clearly, $w^*=1/2$, and $\E[(wx-w^* x)^2] = \E[(w-w^*)^2] =
      (1/2-w)^2$. However, the excess risk equals
$(w-1)^2-(1/2-1)^2 = w^2-2w+3/4= (1/2-w)^2+(1/2-w)$. This is larger than
  the excess risk by an additive factor of $(1/2-w)$, and a multiplicative
  factor of $\frac{1}{1/2-w}$ -- arbitrarily large if $w$ is close to
  $w^*=1/2$.}. Therefore, upper bounds on these measures do not imply upper
  bounds on the excess risk in our setting. We remark that in our
  distribution-free setting, we must constrain the hypothesis class, since
  if our hypothesis class contains all linear predictors ($B=\infty$), then
  the lower bounds below imply that non-trivial learning is impossible with
  any sample size (regardless of the dimension $d$).
  \item \emph{Bounded Functions:} Many learning theory results for the
      squared loss (such as thosed based on fat-shattering techniques)
      assume that the predictor functions and target values are bounded in
      some fixed interval (such as [-1,+1]). In our setting, this would
      correspond to assuming $B,Y\leq 1$. Other results assume Lipschitz
      loss functions, which is not satisfied for the squared loss. One
      notable exception is \citet{srebro2010smoothness}, which analyze
      smooth and strongly-convex losses (such as the squared loss) and
      provide tight bounds. However, their results apply either when the
      functions are bounded by $1$, or when $d$ is extremely large or
      infinite dimensional. In contrast, we provide more general results
      which hold for any $d$ and when the functions are not necessarily
      bounded by $1$.
  \item \emph{Collapsing Problem Parameters Together:} Many results
      implicitly take $Y$ to equal the largest possible prediction,
      $\sup_{\bw,\bx} |\inner{\bw,\bx}|=B$, and give results only in terms
      of $B$. However, we will see that $B$ and $Y$ affect the excess risk
      in a different manner, and it is thus important to discern between
      them. Moreover, $B$ and $Y$ can often have very different magnitudes.
      For example, in learning problems where the instances $\bx$ tend to
      be sparse, we may want to have the norm bound $B$ of the predictor to
      scale with the dimension $d$, while the bound on the target values
      $Y$ remain a fixed constant.
\end{itemize}

\section{Main Result}

Our main result is the following lower bound on the attainable excess risk, for algorithms returning a linear predictor based on an i.i.d. sample:
\begin{theorem}\label{thm:main}
  There exists a universal constant $c$, such that for any dimension $d$, sample size $m$, target value bound $Y$, predictor norm bound $B\geq 2Y$, and for any algorithm returning a linear predictor $\hat{\bw}$, there exists a valid data distribution such that
  \[
  \E[R(\hat{\bw})-R(\bw^*)]\geq   c~\min\left\{Y^2,\frac{B^2+dY^2}{m},\frac{BY}{\sqrt{m}}\right\},
  \]
  where $\bw^*=\arg\min_{\bw:\norm{\bw}\leq B}R(\bw)$.
\end{theorem}
The lower bound is a minimum of three terms. The first and last term can be matched up to constants using existing results:
\begin{itemize}
	\item Using the trivial zero predictor $\hat{\bw}=\mathbf{0}$, we are
	guaranteed that $\E[R(\hat{\bw})-R(\bw^*)]\leq \E[R(\hat{\bw})] =
	\E[(\inner{\mathbf{0},\bx}-y)^2] = \E[y^2]\leq Y^2$.
	\item Alternatively, by corollary 3 in \citet{srebro2010smoothness}
	\footnote{Where $\bar{L^*}\leq Y^2$ and $H=2$ for the squared
		loss.}, using mirror descent with an online-to-batch conversion
	gives us an algorithm for which $\E[R(\hat{\bw})-R(\bw^*)]\leq
	\Ocal\left(\frac{BY}{\sqrt{m}}+\frac{B^2}{m}\right)$. In the regime
	where this bound is smaller than $Y^2$, it can be verified that
	$BY/\sqrt{m}$ is the dominant term, in which case we get an
	$\Ocal(BY/\sqrt{m})$ bound.
\end{itemize}
As to the second $\frac{B^2+dY^2}{m}$ term, it can matched (or even surpassed) by online learning algorithms based on the Vovk-Azoury-Warmouth forecaster, together with a standard online-to-batch conversion (\citet{vovk2001competitive,azoury2001relative,cesa2004generalization,vavskevivcius2020suboptimality}). However, these algorithms are \emph{improper}, in the sense that they do not return a linear predictor, and thus do not fall within the setting of \thmref{thm:main}. In the original version of this note, we conjectured that our lower bound can be matched via empirical risk minimization (i.e., given a training set $\{(\bx_i,y_i)\}_{i=1}^{m}$, return the linear predictor
$\hat{\bw}=\min_{\bw:\norm{\bw}\leq
	B}\frac{1}{m}\sum_{i=1}^{m}(\inner{\bw,\bx_i}-y_i)^2$). However, this was recently disproved in \citet{vavskevivcius2020suboptimality}.
%
%Based on existing results in the literature, this bound has essentially
%matching upper bounds, up to logarithmic factors:
%\begin{itemize}
%    \item Using the trivial zero predictor $\hat{\bw}=\mathbf{0}$, we are
%        guaranteed that $\E[R(\hat{\bw})-R(\bw^*)]\leq \E[R(\hat{\bw})] =
%        \E[(\inner{\mathbf{0},\bx}-y)^2] = \E[y^2]\leq Y^2$.
%    \item Using the Vovk-Azoury-Warmuth forecaster and a standard
%        online-to-batch conversion
%        (\citet{vovk2001competitive,azoury2001relative,cesa2004generalization}),
%        we have an algorithm for which $\E[R(\hat{\bw})-R(\bw^*)]\leq
%        \Ocal\left( \frac{B^2+dY^2\log(1+m/d)}{m}\right)$.
%    \item Alternatively, by corollary 3 in \citet{srebro2010smoothness}
%        \footnote{Where $\bar{L^*}\leq Y^2$ and $H=2$ for the squared
%        loss.}, using mirror descent with an online-to-batch conversion
%        gives us an algorithm for which $\E[R(\hat{\bw})-R(\bw^*)]\leq
%        \Ocal\left(\frac{BY}{\sqrt{m}}+\frac{B^2}{m}\right)$. In the regime
%        where this bound is smaller than $Y^2$, it can be verified that
%        $BY/\sqrt{m}$ is the dominant term, in which case we get an
%        $\Ocal(BY/\sqrt{m})$ bound.
%\end{itemize}
%Taking the best of these algorithmic approaches, we get the minimum of these
%upper bounds, i.e. we can find a predictor $\hat{\bw}$ for which
%\[
%\E[R(\hat{\bw})-R(\bw^*)]\leq   \Ocal\left(\min\left\{Y^2,\frac{B^2+dY^2\log\left(1+\frac{m}{d}\right)}{m},\frac{BY}{\sqrt{m}}\right\}\right).
%\]
%We conjecture that the same bound can be shown for empirical risk
%minimization .

Our lower bound has some interesting consequences: First, it implies that even
when $d=1$ (i.e. a one-dimensional problem), there is a non-trivial
dependence on the norm bound $B$. This is in contrast to results under the
well-specified model or other common distributional assumptions, which lead
to upper bounds independent of $B$. Second, it shows that in a
finite-dimensional setting, although the squared loss $(\inner{\bw,\bx}-y)^2$
may appear symmetric with respect to $y$ and $\inner{\bw,\bx}$, the
attainable excess risk may actually be much more sensitive to the bound $Y$ on
$|y|$ than to the bound $B$ on $|\inner{\bw,\bx}|$, due to the $d$ factor.
For example, if $Y$ is a constant, then $B$ can be as large as the dimension
$d$ without affecting the leading term of the excess risk. 
%Third, in the
%context of online learning, it implies that the Vovk-Azoury-Warmuth
%forecaster is essentially optimal in our setting and for a finite-dimensional
%regime, in terms of its dependence on both $d$ and $B$ (the lower bounds in
%\citet{vovk2001competitive,singer2002universal} do not show an explicit
%dependence on $B$).

\section{Proof of \thmref{thm:main}}

The proof of our main result consist of two separate lower bounds, each of
which uses a different construction. The theorem follows by combining them
and performing a few simplifications.

We begin by recalling the following result, which follows from the well-known
orthogonality principle:
\begin{lemma}\label{lem:proj}
  Let $R(\bw)=\E[(\inner{\bw,\bx}-y)^2]$, and $\bw^*=\arg\min_{\bw:\norm{\bw}\leq B}R(\bw)$. Then for any $\bw\in\reals^d$, it holds that
  \[
  R(\bw)-R(\bw^*)\geq \E[(\inner{\bw,\bx}-\inner{\bw^*,\bx})^2],
  \]
  with equality when $B=\infty$
\end{lemma}
\begin{proof}[Proof Sketch]
  For any $\bw\in \reals^d$, define the linear function $f_{\bw}:\reals^d\mapsto \reals$ by $f_{\bw}(\bx)=\inner{\bw,\bx}$. Then $\{f_\bw(\cdot):\norm{\bw}\leq B\}$ corresponds to a closed convex set in the $L^2$ function space defined via the inner product $\inner{f,g}=\E_{\bx}[f(\bx)g(\bx)]$ and norm $\norm{f}^2=\E_{\bx}[f^2(\bx)]$. Moreover, letting $\eta(\bx)=\E[y|\bx]$, we have
  \[
  R(\bw)-R(\bw^*) = \E[(\inner{\bw,\bx}-y)^2]-\E[(\inner{\bw,\bx}-y)^2] = \E[(f_{\bw}(\bx)-\eta(\bx))^2]-\E[(f_{\bw^*}(\bx)-\eta(\bx))^2] = \norm{f_{\bw}-\eta}^2-\norm{f_{\bw^*}-\eta}^2.
  \]
  In this representation, the inequality in the lemma reduces to
  \[
  \norm{f_{\bw}-f_{\bw^*}}^2+\norm{f_{\bw^*}-\eta}^2 \leq \norm{f_{\bw}-\eta}^2.
  \]
  When $B=\infty$, then $f_{\bw^*}$ is the projection of $\eta$ on the linear sub-space of linear functionals, hence the inequality above holds with equality by the pythagorean theorem. When $B<\infty$, then $f_{\bw^*}$ is the projection of $\eta$ on a constrained subset of this linear space, and we only have an inequality.
\end{proof}

Our first construction provides an excess risk lower bound even when we deal
with one-dimensional problems:
\begin{theorem}\label{thm:one}
  There exists a universal constant $c$, such that for any sample size $m$, target value bound $Y$, predictor norm bound $B\geq 2Y$, and any algorithm returning a linear predictor $\hat{\bw}$, there exists a data distribution in $d=1$ dimensions such that
  \[
  \E[R(\hat{w})-R(w^*)] \geq c\min\left\{Y^2,\frac{B^2}{m}\right\}.
  \]
\end{theorem}
\begin{proof}
  Let $\alpha,\gamma$ be small positive parameters in $(0,1]$ to be chosen later, such that $\alpha>\gamma$, and consider the following two distributions over $(x,y)$:
  \begin{itemize}
    \item Distribution $\Dcal_0$: $y=Y$ w.p. 1;
        $x=\begin{cases}Y/B&~\text{w.p.}~\alpha\\0&~\text{w.p.}~1-\alpha\end{cases}$.
    \item Distribution $\Dcal_1$: $y=Y$ w.p. 1;
        $x=\begin{cases}1&~\text{w.p.}~\gamma\\
        Y/B&~\text{w.p.}~\alpha-\gamma\\0&~\text{w.p.}~1-\alpha\end{cases}$.
  \end{itemize}
  Note that since $B\geq 2Y$, $|x|\leq 1$, so these are indeed valid distributions. Intuitively, in both distributions $x$ is small most of the time, but under $\Dcal_1$ it can occasionally have a ``large'' value of $1$. Unless the sample size is large enough, it is not possible to distinguish between these two distributions, and this will lead to an excess risk lower bound.

  Let $\E_0$ and $\E_1$ denote expectations with respect to $\Dcal_0$ and $\Dcal_1$ respectively. Let
  \[
  w^*_0 = B
  \]
  denote the optimal predictor under $\Dcal_0$, and let
  \[
  w^*_1 = \frac{\E_{1}[yx]}{\E_{1}[x^2]} = \frac{(Y^2/B)(\alpha-\gamma)+Y\gamma}{(Y^2/B^2)(\alpha-\gamma)+\gamma}
  = B\frac{Y^2(\alpha-\gamma)+BY\gamma}{Y^2(\alpha-\gamma)+B^2\gamma}
  \]
  denote the optimal predictor under $\Dcal_1$. Note that $w^*_1\geq w^*_0$, and moreover,
  \begin{equation}\label{eq:diff}
  (w^*_1-w^*_0)^2 = B^2\left(\frac{Y^2(\alpha-\gamma)+BY\gamma}{Y^2(\alpha-\gamma)+B^2\gamma}-1\right)^2
  = B^4\gamma^2\left(\frac{Y-B}{Y^2\alpha+(B^2-Y^2)\gamma}\right)^2
  \geq B^4\gamma^2\left(\frac{Y-B}{Y^2\alpha+B^2\gamma}\right)^2
  \end{equation}

  By Yao's minimax principle, it is sufficient to show that when choosing either $\Dcal_0$ or $\Dcal_1$ uniformly at random, and generating a dataset according to that distribution, any deterministic algorithm attains the lower bound in the theorem. Using \lemref{lem:proj}, and the notation $\Pr_0$ (respectively $\Pr_1$) to denote probabilities with respect to $\Dcal_0$ (respectively $\Dcal_1$), we have
  \begin{align*}
    \E\left[R(\hat{w})-R(w^*)\right] &= \frac{1}{2}\left(\E_0[(\hat{w}x-w^*_0x)^2]+\E_1[(\hat{w}x-w^*_1x)^2]\right)\\
    &\geq \frac{1}{2}\frac{Y^2 \alpha}{B^2}\left(\E_0[(\hat{w}-w^*_0)^2]+\E_1[(\hat{w}-w^*_1)^2]\right)\\
    &\geq \frac{1}{2}\frac{Y^2 \alpha}{B^2}\left(\frac{w^*_1-w^*_0}{2}\right)^2\left({\Pr}_0\left(\hat{w}< \frac{w^*_0+w^*_1}{2}\right)+{\Pr}_1\left(\hat{w}\geq \frac{w^*_0+w^*_1}{2}\right)\right)\\
    &= \frac{1}{2}\frac{Y^2\alpha}{B^2}\left(\frac{w^*_1-w^*_0}{2}\right)^2\left(1-\left({\Pr}_0\left(\hat{w}\geq \frac{w^*_0+w^*_1}{2}\right)-{\Pr}_1\left(\hat{w}\geq \frac{w^*_0+w^*_1}{2}\right)\right)\right)\\
    &\geq \frac{1}{2}\frac{Y^2 \alpha}{B^2}\left(\frac{w^*_1-w^*_0}{2}\right)^2\left(1-\left|{\Pr}_0\left(\hat{w}\geq \frac{w^*_0+w^*_1}{2}\right)-{\Pr}_1\left(\hat{w}\geq \frac{w^*_0+w^*_1}{2}\right)\right|\right).
  \end{align*}
  By Pinsker's inequality, since $\hat{w}$ is a deterministic function of the training set $S$, this is at least
  \[
  \frac{1}{8}\frac{Y^2\alpha}{B^2}\left(w^*_1-w^*_0\right)^2\left(1-\sqrt{\frac{1}{2}D_{kl}({\Pr}_0(S)||{\Pr}_1(S))}\right),
  \]
  where $D_{kl}$ is the Kullback-Leibler divergence. Since $S$ is composed of $m$ i.i.d. instances, and the target value $y$ is fixed under both distributions, we can invoke the chain rule and rewrite this as
  \[
  \frac{1}{8}\frac{Y^2\alpha}{B^2}\left(w^*_1-w^*_0\right)^2\left(1-\sqrt{\frac{m}{2}~D_{kl}({\Pr}_0(x)||{\Pr}_1(x))}\right),
  \]
  To simplify the bound, note that the Kullback-Leibler divergence between two distributions $p,q$ can be upper bounded by their $\chi^2$ divergence, which equals $\sum_a \frac{(p(a)-q(a))^2}{q(a)}$. Therefore,
  \[
  D_{kl}({\Pr}_0(x)||{\Pr}_1(x)) \leq \frac{\gamma^2}{\gamma}+\frac{\gamma^2}{\alpha-\gamma} = \gamma\left(1+\frac{\gamma}{\alpha-\gamma}\right).
  \]
  Plugging this back, as well as the value of $(w^*_1-w^*_0)^2$ from \eqref{eq:diff}, we get an excess loss lower bound on the form
  \[
  \frac{1}{8}Y^2\alpha B^2\gamma^2\left(\frac{Y-B}{Y^2\alpha+B^2 \gamma}\right)^2
  \left(1-\sqrt{\frac{m}{2}\gamma\left(1+\frac{\gamma}{\alpha-\gamma}\right)}\right),
  \]
  We now consider two cases:
  \begin{itemize}
    \item If $m\leq B^2/Y^2$, we pick $\alpha=1$ and $\gamma=1/3m$, and get
        that the expression above is at least
    \begin{align*}
    &\frac{Y^2}{72}\frac{B^2}{m^2}\left(\frac{B-Y}{Y^2+B^2/3m}\right)^2
    \left(1-\sqrt{\frac{1}{6}\left(1+\frac{1/3m}{1-1/3m}\right)}\right)\\
    &=
    \frac{Y^2}{72}\left(\frac{B(B-Y)}{mY^2+B^2/3}\right)^2
    \left(1-\sqrt{\frac{1}{6}\left(1+\frac{1}{3m-1}\right)}\right)\\
    &\geq
    \frac{Y^2}{72}\left(\frac{B(B-Y)}{(B^2/Y^2)Y^2+B^2/3}\right)^2
    \left(1-\sqrt{\frac{1}{6}\left(1+\frac{1}{3m-1}\right)}\right)\\
    &\geq
    \frac{Y^2}{72}\left(\frac{B(B-Y)}{(1+1/3)B^2}\right)^2
    \left(1-\sqrt{\frac{1}{6}\left(1+\frac{1}{2}\right)}\right)\\
    &\geq 0.003~Y^2 \left(\frac{B-Y}{B}\right)^2
    ~=~ 0.003~Y^2\left(1-\frac{Y}{B}\right)^2 ~\geq~ 0.003~Y^2\left(1-\frac{1}{2}\right)^2,
    \end{align*}
    where we used the assumption that $B\geq 2Y$.
  \item If $m> B^2/Y^2$, we pick $\alpha=B^2/(Y^2 m)$ and $\gamma=1/3m$ and
      get that the expression above is at least
  \begin{align*}
  &\frac{1}{8}\frac{B^4}{m}\frac{1}{9m^2}\left(\frac{B-Y}{B^2/m+B^2/3m}\right)^2
  \left(1-\sqrt{\frac{1}{6}\left(1+\frac{1/3m}{(B^2/Y^2-1/3)/m}\right)}\right)\\
  &\geq \frac{1}{72}\frac{(B-Y)^2}{m(1+1/3)^2}
  \left(1-\sqrt{\frac{1}{6}\left(1+\frac{1/3}{4-1/3}\right)}\right)\\
  &\geq 0.004 \frac{(B-Y)^2}{m} ~\geq~ 0.004 \frac{(B-B/2)^2}{m} ~=~ 0.001\frac{B^2}{m},
  \end{align*}
  where we used the assumption that $B\geq 2Y$.
  \end{itemize}
  Combining the two cases, we get an excess risk lower bound of $c~\min\left\{Y^2,\frac{B^2}{m}\right\}$ for some universal constant $c$.
\end{proof}

Our second construction provides a different type of bound, which quantifies
a dependence on the dimension $d$. The construction is similar to standard
dimension-dependent lower bounds for learning with the squared loss, but we
are careful to analyze the dependence on all relevant parameters.
\begin{theorem}\label{thm:slow}
  There exists a universal constant $c$, such that for any dimension $d$, sample size $m$, target value bound $Y$, predictor norm bound $B$ and any algorithm returning a linear predictor $\hat{\bw}$, there exists a data distribution in $d$ dimensions such that
  \[
  \E[R(\hat{\bw})-R(\bw^*)]\geq   c~\min\left\{Y^2,B^2,\frac{dY^2}{m},\frac{BY}{\sqrt{m}}\right\}.
  \]
\end{theorem}

\begin{proof}
  By Yao's minimax principle, it is sufficient to display a randomized choice of data distributions, with respect to which the expected excess error of any deterministic algorithm attains the lower bound in the theorem.

  In particular, fix some $d'\leq d$ to be chosen later, let $\bsigma\in \{-1,+1\}^{d'}$ be chosen uniformly at random, and consider the distribution $\Dcal_{\bsigma}$ (indexed by $\bsigma$) over examples $(\bx,y)$, defined as follows: $\bx$ is chosen uniformly at random among the first $d'$ standard basis vectors, and $y=Y$ with probability $\frac{1}{2}\left(1+\sigma_i b\right)$, where $b=\min\{1/2,\sqrt{d'/6m}\}$, and $y=-Y$ otherwise.

  A simple calculation shows that the optimum $\bw^*=\arg\min_{\bw:\norm{\bw}\leq B}R(\bw)$ is such that
  \[
  \forall~i~,~~~w^*_i =  \sigma_i\min\{Yb,B/\sqrt{d}\}.
  \]
  Therefore, using \lemref{lem:proj} and the notation $\mathbf{1}_A$ to be the indicator function for the event $A$:
  \begin{align*}
    \E\left[R(\hat{\bw})-R(\bw^*)\right] &= \E[(\inner{\hat{\bw},\bx}-\inner{\bw^*,\bx})^2]\\
    &= \E\left[\frac{1}{d'}\sum_{i=1}^{d'}(\hat{\bw}_i-\bw^*_i)^2\right]\\
    &= \frac{1}{d'}\sum_{i=1}^{d'}\E[(\hat{\bw}_i-\bw^*_i)^2]\\
    &\geq \frac{1}{d'}\sum_{i=1}^{d'}\E[(\bw^*_i)^2 \mathbf{1}_{\hat{\bw}_i\bw^*_i\leq 0}]\\
    &= \frac{1}{d'}\left(\min\{Yb,B/\sqrt{d'}\}\right)^2\sum_{i=1}^{d'}\Pr(\hat{\bw}_i\bw^*_i\leq 0).
  \end{align*}
  Since $\sigma_i$ is uniformly distributed on $\{-1,+1\}$, and has the same sign as $w^*_i$, this equals
  \begin{align*}
    &\frac{1}{d'}\left(\min\{Yb,B/\sqrt{d'}\}\right)^2\sum_{i=1}^{d'}\frac{1}{2}
    \left(\Pr(\hat{\bw}_i\geq 0|\sigma_i<0)+\Pr(\hat{\bw}_i\leq 0|\sigma_i>0)\right)\\
    &\geq \frac{1}{2d'}\left(\min\{Yb,B/\sqrt{d'}\}\right)^2\sum_{i=1}^{d'}
    \left(1-\Pr(\hat{\bw}_i\leq 0|\sigma_i<0)+\Pr(\hat{\bw}_i\leq 0|\sigma_i>0)\right)\\
    &\geq \frac{1}{2d'}\left(\min\{Yb,B/\sqrt{d'}\}\right)^2\sum_{i=1}^{d'}
    \left(1-\left|\Pr(\hat{\bw}_i\leq 0|\sigma_i<0)-\Pr(\hat{\bw}_i\leq 0|\sigma_i>0)\right|\right)\\
  \end{align*}
  Using Pinsker's inequality and the fact that $\hat{\bw}$ is a deterministic function of the training set $S$, this is at least
  \begin{equation}\label{eq:pinsker}
  \frac{1}{2d'}\left(\min\{Yb,B/\sqrt{d'}\}\right)^2\sum_{i=1}^{d'}
    \left(1-\sqrt{\frac{1}{2}D_{kl}\left(\Pr(S|\sigma_i<0)||\Pr(S|\sigma_i>0)\right)}\right),
  \end{equation}
  where $D_{kl}$ is the Kullback-Leibler (KL) divergence. Since the training set is composed of $m$ i.i.d. instances,  we can use the chain rule and get that this divergence equals $mD_{kl}\left(\Pr((\bx,y)|\sigma_i<0)||\Pr((\bx,y)|\sigma_i>0)\right)$. Moreover, we note that
  \begin{align*}
  \Pr((\bx,y)|\sigma_i)&=~\Pr(\bx=\be_i)\Pr((\bx,y)|\sigma_i,\bx_i=\be_i)+\Pr(\bx\neq \be_i)\Pr((\bx,y)|\sigma_i,\bx\neq \be_i)\\
  &=~\frac{1}{d'}\Pr((\bx,y)|\sigma_i,\bx=\be_i)+\left(1-\frac{1}{d'}\right)\Pr((\bx,y)|\sigma_i,\bx\neq\be_i),
  \end{align*}
  and therefore, by joint convexity of the KL-divergence, we get
  \begin{align*}
  D_{kl}(\Pr((\bx,y)|\sigma_i>0)||\Pr((\bx,y)|\sigma_i<0)) &= \frac{1}{d'}D_{kl}\left(\Pr((\bx,y)|\sigma_i<0,\bx=\be_i)||\Pr((\bx,y)|\sigma_i>0,\bx=\be_i)\right)\\
  &+\left(1-\frac{1}{d'}\right)D_{kl}\left(\Pr((\bx,y)|\sigma_i<0,\bx\neq\be_i)||\Pr((\bx,y)|\sigma_i>0,\bx\neq\be_i)\right).
  \end{align*}
  Since the distribution of $y$ is independent of $\sigma_i$, conditioned on $\bx\neq \be_i$, this equals
  \begin{equation}\label{eq:dkl2}
  \frac{1}{d'}D_{kl}\left(\Pr(y|\sigma_i>0,\bx=\be_i)||\Pr(y|\sigma_i<0,\bx=\be_i)\right).
  \end{equation}
  The divergence in this equation is simply the KL divergence between two Bernoulli random variables, one with parameter $\frac{1}{2}\left(1+b\right)$, and the other with parameter $\frac{1}{2}\left(1-b\right)$. To get a simple upper bound, note that the KL divergence between two distributions $p,q$ can be upper bounded by their $\chi^2$ divergence, which equals $\sum_a \frac{(p(a)-q(a))^2}{q(a)}$. Therefore, we can upper bound \eqref{eq:dkl2} by
  \[
  \frac{b^2}{d'}\left(\frac{1}{\frac{1}{2}(1+b)}+\frac{1}{\frac{1}{2}(1-b)}\right)=
  \frac{2b^2}{d'}\left(\frac{1}{1+b}+\frac{1}{1-b}\right)
  \leq   \frac{2b^2}{d'}\left(1+\frac{1}{1/2}\right) = \frac{6b^2}{d'},
  \]
  where we used the fact that $b\in [0,1/2]$. Summarizing the discussion so far, we showed that
  \[
  D_{kl}\left(\Pr(S|\sigma_i<0)||\Pr(S|\sigma_i>0)\right)~=~ m~D_{kl}\left(\Pr((\bx,y)|\sigma_i<0)||\Pr((\bx,y)|\sigma_i>0)\right)
  ~=~ \frac{6mb^2}{d'}.
  \]
  Plugging this back into \eqref{eq:pinsker}, we get that the excess risk is lower bounded by
  \begin{align*}
  \frac{1}{2d'}\left(\min\{Yb,B/\sqrt{d'}\}\right)^2\sum_{i=1}^{d'}
    \left(1-\sqrt{\frac{3mb^2}{d'}}\right)
    &=
    \left(\min\{Yb,B/\sqrt{d'}\}\right)^2\frac{1}{2}
    \left(1-\sqrt{\frac{3mb^2}{d'}}\right)
    \\
    &\geq
    \left(\min\{Yb,B/\sqrt{d'}\}\right)^2\frac{1}{2}
    \left(1-\sqrt{\frac{3m(d'/6m)}{d'}}\right)\\
    &\geq
    0.14 \left(\min\{Yb,B/\sqrt{d'}\}\right)^2\\
    &=
    0.14 \left(\min\left\{Y\min\left\{\frac{1}{2},\sqrt{\frac{d'}{6m}}\right\},\frac{B}{\sqrt{d'}}\right\}\right)^2\\
    &=
    0.14 \min\left\{\frac{1}{4}Y^2,\frac{d'Y^2}{6m},\frac{B^2}{d'}\right\}.
  \end{align*}
  Now, recall that $d'$ is a free parameter of value at most $d$. We now distinguish between two cases:
  \begin{itemize}
    \item If $d > \sqrt{6m}B/Y $, then we pick $d'= \lceil \sqrt{6m}B/Y
        \rceil$, and get that the expression above is at least
        \[
        0.14 \min\left\{\frac{1}{4}Y^2,\frac{B^2}{d'}\right\}
        ~\geq~
        0.14 \min\left\{\frac{1}{4}Y^2,\frac{B^2}{\max\left\{1,2\sqrt{6m}\frac{B}{Y}\right\}}\right\}
        ~=~
        0.14 \min\left\{\frac{1}{4}Y^2,B^2,\frac{BY}{2\sqrt{6m}}\right\}.
        \]
    \item If $d \leq \sqrt{6m}B/Y $, we pick $d'=d$, and note that
        $\frac{d'Y^2}{6m}\leq \frac{B^2}{d}$ in this case. Therefore, the
        expression above is at least
    \[
    0.14 \min\left\{\frac{1}{4}Y^2,\frac{dY^2}{6m}\right\}
    \]
  \end{itemize}
  Combining the two cases, we get that a lower bound of the form
  \[
  c~\min\left\{Y^2,B^2,\frac{dY^2}{m},\frac{BY}{\sqrt{m}}\right\},
  \]
  where $c$ is a universal constant.
\end{proof}

With \thmref{thm:one} and \thmref{thm:slow} at hand, we now turn to prove our
main result:
\begin{proof}[Proof of \thmref{thm:main}]
  Taking the maximum of \thmref{thm:one} and \thmref{thm:slow}, and using the fact that $B\geq 2Y$, we get a lower bound of
  \[
  c\max\left\{\min\left\{Y^2,\frac{B^2}{m}\right\}~,~\min\left\{Y^2,\frac{dY^2}{m},\frac{BY}{\sqrt{m}}\right\}\right\}
  \]
  for some constant $c$. If $m\leq (B^2/Y^2)$, this is at least $Y^2$, and otherwise it is
  \[
  c\max\left\{\frac{B^2}{m}~,~\min\left\{\frac{dY^2}{m},\frac{BY}{\sqrt{m}}\right\}\right\}
  ~\geq~ \frac{c}{2}\left(\frac{B^2}{m}+\min\left\{\frac{dY^2}{m},\frac{BY}{\sqrt{m}}\right\}\right)
  ~\geq~ \frac{c}{2}\min\left\{\frac{B^2+dY^2}{m},\frac{BY}{\sqrt{m}}\right\}.
  \]
  Combining the two cases, the result follows.
\end{proof}

\paragraph{Acknowledgements:}{We thank Nati Srebro, Tomas  Va{\v{s}}kevi{\v{c}}ius and Nikita Zhivotovskiy for very helpful comments.}

\bibliographystyle{plainnat}
\bibliography{mybib}

\end{document}